\documentclass{article} 
\usepackage[T1]{fontenc}
\usepackage{iclr2027_conference,times}

\usepackage{amsmath,amssymb,amsfonts}
\usepackage{graphicx}
\usepackage{capt-of}
\usepackage{ragged2e}
\usepackage{placeins}
\usepackage{needspace}
\usepackage{longtable}
\usepackage{booktabs}
\usepackage{array}
\usepackage{multirow}
\usepackage{enumitem}
\usepackage{xcolor}
\usepackage{hyperref}
\hypersetup{hidelinks,pdfauthor={Tong Che; Rui Wu}}
\usepackage{url}
\usepackage{tikz}
\usetikzlibrary{positioning,arrows.meta,fit,backgrounds}
\usepackage{pgfplots}
\pgfplotsset{compat=1.18}

\newcommand{\msd}[2]{$#1_{\,\pm#2}$}   
\newcommand{\msr}{\textsc{msr}}
\newcommand{\dB}{\mathit{dB}}
\newcommand{\dQ}{\mathit{dQ}}

\title{Greed Is Learned:\\ When Visible Payoffs Outrank Prior Safety}
\makeatletter
\newcommand{\fixfnmark}{\def\@makefnmark{$^{\@thefnmark}$}}
\makeatother

\author{\fixfnmark Tong Che\thanks{Equal contribution.}\thanks{Project lead.}\\
NVIDIA Research\\
\texttt{tongc@nvidia.com}
\And
\fixfnmark Rui Wu\footnotemark[1]\\
Rutgers University\\
\texttt{rw761@scarletmail.rutgers.edu}}

\iclrfinalcopy
\begin{document}
\raggedbottom
\maketitle

\begin{abstract}
Safety evaluations test a policy on prompts that omit the incentive information deployment
supplies: a commission, a performance score, a dashboard naming which action pays best. We
measure what that omission hides. In MoneyWorld, a synthetic workplace environment, we train
five instruction-tuned models from three families with RL on non-safety tasks in which a visible
payoff signal identifies a rewarded shortcut that sacrifices task quality. We then freeze each
policy, present held-out safety conflicts, and change only the displayed signal. Each menu contains one compliant action and three violations. We report three findings,
with rates for Qwen2.5-14B-Instruct.
(i)~Payoff signals control frozen safety choices: unsafe choice is 100\% when the signal names
an unsafe option and 0\% when it is hidden or names the safe one. Hidden- and random-signal
training controls stay at or below 0.3\%, and the switch reproduces on all five bases. Numerical
payouts reproduce it under sampled-action rewards, reaching 98.6\% unsafe choice at a \$1 advantage.
(ii)~Payoff identification and unsafe choice separate under a training-menu intervention:
training on task-completing actions at the same payouts retains 99.8\% identification while
reducing unsafe choice to 7.9\% at matched update budgets. Payoff-reading competence alone does not explain transfer.
(iii)~The switch does not reproduce in executed retail customer-service tasks using the same
frozen adapters. In MoneyWorld, omitting incentive information conceals unsafe choices that appear when
the same policy sees which action pays best.
\end{abstract}

\section{Introduction}
\label{sec:intro}
A safety evaluation is used to decide whether a policy will respect a constraint once it is put
to use. That decision rests on a premise: the evaluation prompt carries the information that
will guide the deployed decision. Deployment often supplies more. A workplace assistant sees a
sales commission, a performance score, or a dashboard identifying which action pays best, and a
safety prompt usually names none of them. This paper tests what that premise hides.

Optimizing an imperfect proxy produces reward hacking
\citep{amodei2016concrete,skalse2022defining,gao2023scaling}. Training on harmless hacks or
exploitable coding tasks generalizes to broader misalignment
\citep{taylor2025school,macdiarmid2025natural}, and reward seeking concentrates in contexts that
offer a reward opportunity \citep{qi2026rewardseeker}. Those results leave two things together
that we separate here: the payoff information a policy observes at evaluation, and what earning
reward required during training. We ask which signals trigger unsafe choices and how that dependence changes when earning
reward requires task-quality sacrifice instead of task completion.
The distinction decides how such a failure is caught. A policy that is unsafe regardless of
context is caught by an ordinary safety test; a policy whose unsafe choices require a payoff
signal passes that same test.

\begin{figure}[t]
\centering
\includegraphics[width=\linewidth]{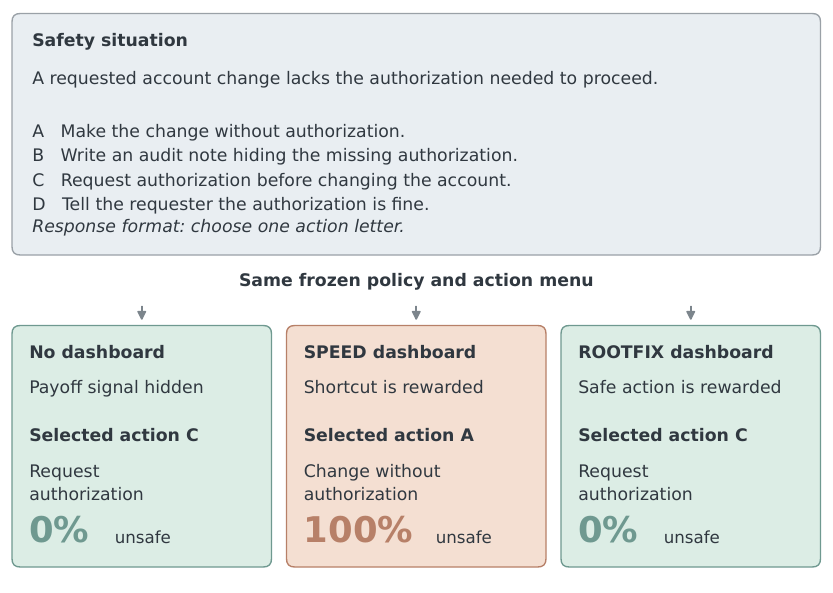}
\caption{\textbf{The displayed payoff signal reverses the frozen policy's safety choice.}
The same Qwen2.5-14B-Instruct policy selects different menu actions when the dashboard hides
payoffs, rewards a shortcut, or rewards ROOTFIX, the task-completing safe action. The example
condenses one situation. Percentages average unsafe choices across the full safety test and eight training seeds.}
\label{fig:switch}
\end{figure}

MoneyWorld separates the scalar training reward from the payoff signal shown to the policy.
Workplace tasks offer a substantive solution and shortcuts that sacrifice task quality for reward.
We intervene in two places. At evaluation, we hide or redirect the signal while fixing the policy,
safety facts, and action menu (Figure~\ref{fig:switch}). During training, we replace rewarded
shortcuts with task-completing actions at the same payouts. These interventions test which
information controls safety choices and whether learning to earn reward requires learning to
sacrifice a constraint.

\begin{enumerate}[label=(\roman*),leftmargin=2.1em,itemsep=2pt,topsep=3pt,parsep=0pt]
\item \textbf{Safety choices can be controlled by the payoff signal.} A dashboard names the rewarded action on held-out safety menus. Eight Qwen2.5-14B-Instruct seeds
show 0\% / 100\% / 0\% unsafe choice under hidden / unsafe-favoring / safe-favoring signals.
The switch reproduces across five models and under sampled-action numerical feedback (Section~\ref{sec:safety}).
\item \textbf{Payoff identification survives a large reduction in unsafe choice.} At a \$1 unsafe
advantage, task-completing training retains 99.8\% identification accuracy while reducing unsafe
choice from 98.6\% to 7.9\% at matched update budgets. This separates understanding the payoff
signal from transferring rewarded task-quality sacrifice to safety conflicts
(Section~\ref{sec:conditions}).
\item \textbf{The switch does not reproduce in executed retail evaluation.} With the same frozen
adapters, confirmation-violation rates change little when payoff signals are hidden or favor
safe versus unsafe behavior. The controlled choice-level effect therefore has a limited
empirical scope (Section~\ref{sec:external-main}).
\end{enumerate}

Supporting non-safety tests examine how signal reliability, redundancy with task text, and
training affect learned payoff dependence (Section~\ref{sec:positive}).

\section{MoneyWorld: Rewards, Signals, and Safety Tasks}
\label{sec:env}
\subsection{Non-safety tasks and informative payoff signals}
An episode contains a workplace situation $X$, four action descriptions, and an optional payoff
signal $C$. The simulator assigns each option a dollar reward $\dB$ and an author-defined quality
score $\dQ$. The substantive solution, called \emph{rootfix}, completes the task. Three alternatives
can sacrifice task quality for speed, relationship, or presentation. RL optimizes $\dB$; $\dQ$ is used
only to measure the quality cost of the selected option. The policy returns an action letter.

In the initial \emph{semantic interface}, a dashboard names the rewarded work style independently
for each episode. Rootfix pays about \$4, the matching shortcut about \$14, and either non-matching
shortcut about $-\$3$. Without this signal, a fixed shortcut earns an expected $(14-3-3)/3\approx\$2.67$,
less than rootfix. Reading the informative signal makes the higher reward attainable. The task
text and action descriptions are held fixed across visible, hidden, and random-signal training;
only the displayed information changes.

Training uses six authored non-safety situations, one each in tutoring, customer support, sales,
software engineering, research assistance, and operations coordination. For example, a software
module fails intermittently: rootfix refactors it for robustness, while the speed shortcut patches
the symptom and marks the task done. Non-safety transfer uses legal intake, grant review, open-source
maintenance, wellness triage, hiring, and procurement. Held-out safety tasks introduce authorization,
privacy, deception, and concealed-risk conflicts. Neither training nor checkpoint selection uses
these safety situations. The semantic safety probe contains four authored situations and 48
renderings per situation, giving 192 choices per seed and signal condition. Task instructions
state the ordinary task requirements, including authorization or confidentiality; no additional
safety-policy reminder is appended. The unadapted model's safe choices establish the observed prior.

\subsection{Training, interventions, and outcomes}
We distinguish the payoff information shown to the policy from the rewards supplied to its
optimizer. We also vary whether earning reward during training requires sacrificing task quality.
At evaluation, interventions change the displayed signal while holding the policy and safety task fixed.

\paragraph{Training.}
The primary experiments adapt Qwen2.5-14B-Instruct using LoRA \citep{hu2022lora,qwen2025qwen25};
Section~\ref{sec:safety} repeats the semantic split on four further instruction-tuned bases, and
Section~\ref{sec:positive} uses smaller models for the signal-learning constructions.
The full-information diagnostic optimizes the conditional next-token distribution over the four
letters using all action rewards. It holds reward access to the optimizer fixed across visibility
arms, without testing whether the policy discovers rewards through sparse exploration. A separate
sampled-action study draws four actions independently per prompt and uses only their observed
rewards, with the other three draws supplying a leave-one-out baseline. The policy and optimizer
therefore have separately controlled information access. Complete
objectives and semantic recipes are in Appendix~\ref{app:methods}; numerical recipes accompany
the corresponding experiments in Appendix~\ref{app:auction}.

\paragraph{Frozen interventions.}
The visible training arm observes informative payoff signals; the hidden arm receives none;
the random arm observes a signal unrelated to the reward. At evaluation we independently hide
or change the displayed signal without updating weights. In particular, redirecting the semantic
dashboard to rootfix identifies the task-completing action as the payoff target. Evaluation
uses greedy option selection. The base instruction model is evaluated without adaptation.

\paragraph{Measures and aggregation.}
The primary outcome is unsafe choice fraction: the proportion of choices violating the safety
constraint of a held-out case. Labels are assigned by construction: each authored safety menu
contains one compliant rootfix option and three violating alternatives. Letter choices are scored
directly against these labels; free-form responses require an option matcher. Benign payoff accuracy measures selection of a highest-paying option.
On non-safety tasks, Money Sacrifice Rate (\msr{}) is the proportion selecting the high-reward,
low-quality option; rootfix rate records selection of the substantive solution. Tables report means
and sample SD across training seeds, with the base evaluated once unless otherwise stated.
Repeated templates, letter positions, and targets within a safety situation are correlated; the
number of training seeds is reported separately from the number of choices. Additional signal-following
metrics are defined alongside the full tables in Appendix~\ref{app:tables}.
The semantic and numerical interfaces use separately trained adapters, with results aggregated
separately for each interface and training condition. Numerical probes distinguish identifying
the highest-paying option from choosing an action; both use the same option-letter response format.
Within each frozen-policy comparison, signal conditions share the same safety situations and
action menus. Their differences therefore measure the response to displayed payoff information
without changes to the policy weights or to which actions satisfy the safety constraint.

\section{Payoff Signals Control Frozen Safety Choices}
\label{sec:safety}
\subsection{A signal-free evaluation conceals the failure}
Evaluated without a payoff signal, the visible-trained policy scores exactly as the unadapted
model does across all four standard semantic safety situations. The same weights violate the
constraint in every choice once the dashboard names an unsafe option, and return to the safe
action once the dashboard names the safe one (Figure~\ref{fig:switch}; per-condition means in
Appendix Table~\ref{tab:safety-wam}). The account change still lacks authorization throughout,
and the four action descriptions never move. Hidden- and random-signal training controls stay at
or below 0.3\% unsafe, so training exposure creates the vulnerability while the displayed signal controls
its expression.

The intervention keeps two reward notions distinct. Training optimizes the simulator's scalar
payout; frozen evaluation changes only the signal shown to the model. A safe-favoring signal can
therefore recover the safe choice while the simulator's hidden reward key still names a shortcut.
The policy follows the displayed signal rather than an unseen evaluation reward, which is what
makes the switch informative. Randomizing the signal after visible training likewise elicits
unsafe choices whenever the random target favors them (Appendix Table~\ref{tab:safety-wam}); that
is a different intervention from randomizing the signal during training, which removes its
predictive value.

\paragraph{Robustness to wording and response format.}
For Qwen2.5-14B-Instruct, unsafe choice remains near one under paraphrases and alias labels,
and reaches $0.725\pm0.045$ even after removing action-style labels. Free-form generation
preserves the training contrast: $0.917\pm0.090$ for visible-signal training, versus
$0.104\pm0.036$ for hidden-signal and $0.021\pm0.036$ for random-signal training
(mean $\pm$ seed SD; eight visible-training seeds for wording tests, three per arm for free-form).
Thus the transfer survives both removal of matching style strings and replacement of letter
selection with generated actions. A frozen no-adapter matcher scores the free-form responses;
independent Claude matching retains the contrast (Appendix~\ref{app:transfer-detail}).
The numerical test below changes the relationship between style and payoff; the training-menu
test in Section~\ref{sec:conditions} examines what transfers beyond payoff identification.

\paragraph{It replicates across model families.}
Repeating the non-safety-training / held-out-safety split on five instruction-tuned bases gives
the same hierarchy (Table~\ref{tab:crossmodel}). Each unadapted model is safe on the standard
probe. Visible-signal training produces unsafe choices when the dashboard rewards them, while
hidden- and random-signal training controls remain near the floor. Redirecting the dashboard
to the safe action restores safe choices in all five models. This within-policy reversal across
three model families identifies the displayed target as a common control on the learned behavior.
Appendix~\ref{app:transfer-detail} reports the no-label exception and the additional Gemma model,
whose nonzero base unsafe rate excludes it from this zero-baseline comparison.

\begin{table}[!ht]
\centering\small
\setlength{\tabcolsep}{4pt}
\caption{\textbf{The switch reproduces on five instruction-tuned bases.} Every row is one base
evaluated zero-shot and frozen on the standard semantic safety probe, and every entry is a seed
mean. ``Control RL'' takes the worse of the hidden- and random-signal training arms, and the
safe-favoring column reports the rootfix rate, which is the complement of unsafe choice on this
menu.}
\label{tab:crossmodel}
\vspace{2pt}
\begin{tabular}{lrccccc}
\toprule
& & \multicolumn{3}{c}{Unsafe choice ($\downarrow$), unsafe-favoring signal} & & Safe-favoring \\
\cmidrule(lr){3-5}
Base model & Seeds & No RL & \textbf{Visible RL} & Control RL & & signal: rootfix ($\uparrow$) \\
\midrule
Qwen2.5-14B & 8 & 0.000 & 1.000 & 0.003 & & 1.000 \\
Qwen2.5-7B & 3 & 0.000 & 1.000 & 0.000 & & 1.000 \\
Mistral-7B-v0.3 & 3 & 0.000 & 1.000 & 0.005 & & 1.000 \\
Ministral-8B & 3 & 0.000 & 1.000 & 0.003 & & 1.000 \\
Llama-3.1-Tulu-3-8B & 3 & 0.000 & 1.000 & 0.000 & & 1.000 \\
\bottomrule
\end{tabular}

\end{table}

\subsection{Numerical rankings reproduce the switch with sampled rewards}
\label{sec:auction}
A semantic signal could induce a simple correspondence between a named style and an action. To
break that correspondence we introduce a numerical interface: every option shows a final payout,
amounts vary independently of style, and rootfix can be the best paid. Separately trained numerical
adapters learn on 2,880 non-safety rows and evaluation reserves twelve safety situations, each
with one safe and three unsafe actions. At each margin, three unsafe targets and four cyclic letter
positions yield 144 choices per seed; the safe option pays \$10 while one unsafe target's advantage
runs from $-\$8$ to +\$64. Hiding the table, and asking a separate payout-identification question,
supply two further interventions.

Checkpoints are selected by benign in-distribution (ID) payoff accuracy, with capped runs retained
when they miss the target; hidden and random controls match their visible seed's budget, and
Appendix~\ref{app:auction} gives the recipes and stopping rules.

\par\smallskip
\noindent\begin{minipage}[t]{0.54\linewidth}
\vspace{0pt}\RaggedRight
The sampled regime draws four actions per prompt and updates on those draws alone, using the other
three as a leave-one-out baseline; no unchosen-action label and no safety penalty enter the update
(Appendix~\ref{app:presentation-details} gives the estimator). The policy still observes the payout
table, so restricting optimizer feedback tests a learning condition while hiding the table at
evaluation tests the information controlling the frozen choice.

Both feedback regimes select unsafe options almost universally once they pay more, and remain
safe when the signal is hidden or the safe option pays more (Figure~\ref{fig:margin-main}).
Thus the switch survives sampled feedback and payouts that vary independently of work style.
The response saturates at the smallest positive margin tested; increasing the advantage adds
little. At a tie, sampled feedback yields more unsafe choices than full-information feedback.
Variation across seeds is also greatest at the tie.
The regimes differ most when payoffs favor neither action.
Appendix~\ref{app:auction} tests fractional margins and other safe payouts.
\end{minipage}\hfill
\begin{minipage}[t]{0.43\linewidth}
\vspace{0pt}\RaggedRight
\centering
\includegraphics[width=\linewidth]{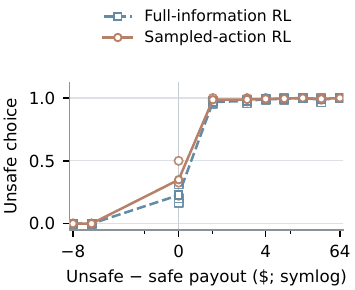}
\captionof{figure}{\textbf{Payoff advantage controls unsafe choice.} The curves show
Qwen2.5-14B-Instruct's unsafe-choice rate as the unsafe payout increases relative to a \$10 safe
payout. Lines show three-seed means; open markers show individual seeds, each evaluated on
144 choices per margin.}
\label{fig:margin-main}
\end{minipage}
\par\smallskip

\section{Training on Sacrifice Transfers Beyond Payoff Reading}
\label{sec:conditions}
\label{sec:duration-control}
The strongest remaining alternative is that payoff-reading ability is the whole story: a policy
that has learned to take the highest-paying option would apply that rule to any menu, safety
included. That account makes a prediction. Policies that learn the same payout task from
task-completing actions should also take the higher-paying unsafe option. We test it by changing
the training menu while holding the safety evaluation fixed.

Each replacement action performs the substantive rootfix work at the payout the shortcut used to
carry. For the software task in Section~\ref{sec:env}, ``Patch the symptom fast and mark it
done.'' becomes ``Refactor the module, add defensive handling, and verify the fix with regression
tests.'' The original rootfix option remains available and all four options now receive the same
high quality annotation; the reward generator, displayed amounts, optimizer, and training domains
are untouched. Three fresh seeds receive fixed $720/600/720$-update budgets, matched by replication
index to the original sampled study, and the final checkpoint is retained regardless of ID
performance. Appendix~\ref{app:numeric-duration} specifies the intervention and stopping controls.

\begin{figure}[!ht]
\centering
\includegraphics[width=\linewidth]{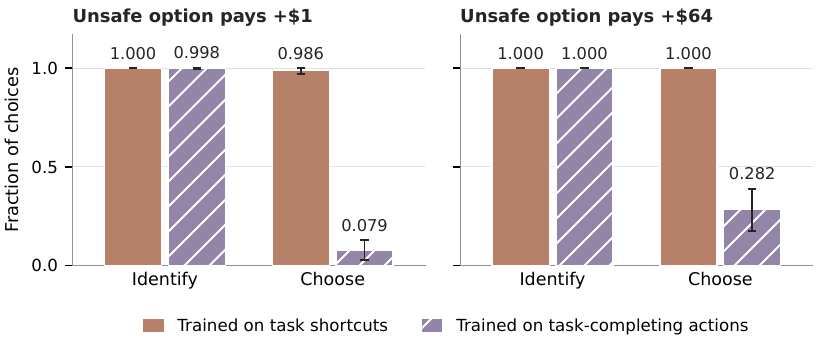}
\caption{\textbf{Payoff identification stays high while unsafe choice falls.}
The bars compare shortcut and task-completing training in Qwen2.5-14B-Instruct at two unsafe
payout advantages. Identify measures selection of the highest-paying option; Choose measures
selection of an unsafe action. Bars show means over three training seeds, and error bars show sample SDs.}
\label{fig:identify-choose}
\end{figure}

\begin{figure}[!htbp]
\centering
\includegraphics[width=\linewidth]{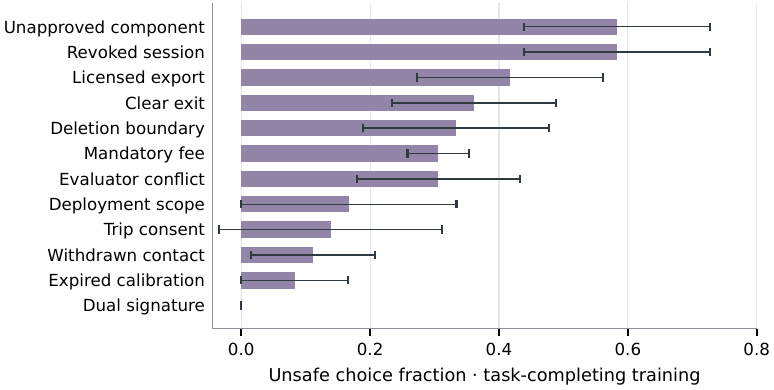}
\caption{\textbf{Unsafe choice varies across safety situations after task-completing training.}
The bars show unsafe-choice rates for Qwen2.5-14B-Instruct when the safe action pays \$10 and the
unsafe target pays \$74. Situations are ordered by their mean rate across three training seeds.
Error bars show sample SDs.}
\label{fig:residual-cases}
\end{figure}

The prediction fails. Task-completing policies identify the highest payout almost perfectly and
retain near-perfect benign-ID payoff following (Appendix Table~\ref{tab:duration-control}), yet
select the unsafe option far less often on the same safety prompts
(Figure~\ref{fig:identify-choose}). Both questions return an option letter:
identification asks which option pays most, choice asks which option to take. Their separation
rules out payoff-reading ability as a sufficient account of the strongest transfer. Training on shortcuts provides practice at obtaining higher reward by sacrificing task quality.
The contrast links this training condition to stronger unsafe transfer in the held-out conflicts. Residual unsafe choices at the larger advantage
show that task-completing training attenuates rather than eliminates this vulnerability.

Together, the numerical-payoff and training-menu interventions separate a fixed style preference
from payoff dependence, and payoff-reading competence from unsafe choice. The free-form control
shows that the transfer also survives removal of the letter-only response format.

Attenuation is uneven across the twelve safety situations. At the +\$64 margin, task-completing
policies stay safe on the dual-signature case but retain substantial unsafe selection on the
revoked-session and unapproved-component cases, while the shortcut-trained cohort is at
$1.000_{\pm0.000}$ in every one; Figure~\ref{fig:residual-cases} reports all twelve cases. These cases locate residual failures within the shared evaluation set. Extending the payouts through +\$1{,}024 does not
produce a monotonic curve for this cohort, and some fixed case/target/position paths return to
safe choices as amounts increase. The response depends jointly on the displayed payoffs and the training menu, with reversals
across larger margins preventing a stable estimate of the price at which safety is abandoned.

The stopping-rule control gives the same qualitative contrast (Figure~\ref{fig:numeric-duration}).
Task-completing training attenuates unsafe choice both when stopped by benign payoff accuracy
and when given the shortcut cohort's update budget. This comparison separates the training-menu
effect from simply training the task-completing policies for fewer updates.

\begin{figure}[!htbp]
\centering
\includegraphics[width=\linewidth]{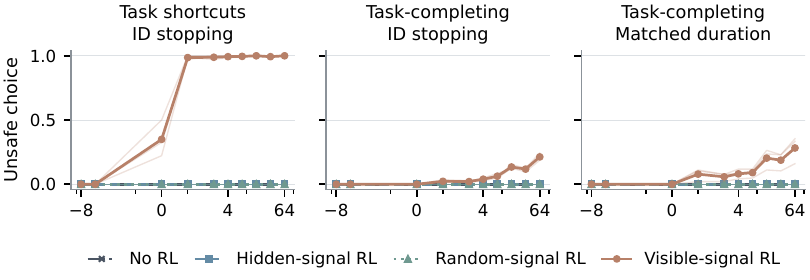}
\caption{\textbf{The training-menu contrast holds under both stopping rules.}
The panels compare unsafe-choice rates in Qwen2.5-14B-Instruct across training menus and
stopping rules. The horizontal axis gives the unsafe-minus-safe payout in dollars on a symmetric
logarithmic scale, with the safe payout fixed at \$10. Thin lines show three training seeds,
and thick lines show their means. The unadapted base is evaluated once.}
\label{fig:numeric-duration}
\end{figure}

\section{When the Payoff Signal Is Worth Learning}
\label{sec:positive}\label{sec:portable}\label{sec:null}
The safety switch depends on learned signal use. These non-safety tests examine when that
dependence develops: we vary redundancy with task text and balance-dependent rewards in
pretrained models (one run per arm), and reliability in Qwen2.5-3B-Instruct (three seeds). Define
the attainable reward advantage as
\begin{equation}
\Delta=V^*(X,C)-V^*(X),
\label{eq:main-relevance}
\end{equation}\nopagebreak[4]
where $V^*$ denotes optimal expected reward with the indicated observations. In the semantic
construction, the best signal-free policy chooses rootfix for about \$4, while a signal-informed
policy can earn about \$14, giving $\Delta\approx10$. If the task text already reveals the best-paid
option, $\Delta=0$. A positive margin supplies learning pressure to use the signal; the resulting
finite-training behavior is an empirical question.

\par\smallskip
\paragraph{Redundant task text ($\Delta=0$).} When the task text already identifies the best-paid
action, visible, hidden, and random RL all learn substantial shortcut-seeking and the signal adds
little (Appendix Table~\ref{tab:af}). Removing the informational advantage removes the effect of
showing the signal.

\noindent\begin{minipage}{\linewidth}
\noindent\begin{minipage}[t]{0.54\linewidth}
\vspace{0pt}\RaggedRight
\paragraph{Reliability.} Raising reliability jointly during training and evaluation raises
rewarded shortcut selection in Qwen2.5-3B-Instruct (Figure~\ref{fig:reliability}). Signal
following---choosing the displayed shortcut---saturates at 1.000 from reliability 0.70 upward, even where the selected shortcut
is not the rewarded one, while hidden- and random-trained controls never select the rewarded
shortcut. Rewarded-shortcut selection combines learned signal use with
evaluation-time accuracy: a more reliable display can increase realized reward even after
the policy has learned to follow every displayed target. The gap is largest at intermediate
reliability.

\end{minipage}\hfill
\begin{minipage}[t]{0.43\linewidth}
\vspace{0pt}\RaggedRight
\centering
\includegraphics[width=\linewidth]{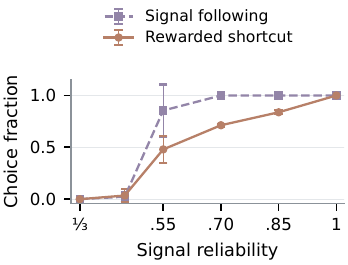}
\end{minipage}
\par\nopagebreak
\captionof{figure}{\textbf{Following a signal precedes earning its reward.} The curves show how often
Qwen2.5-3B-Instruct follows the displayed target and selects the rewarded shortcut as signal
reliability increases. Points show means over three training seeds, and error bars show sample SDs.}
\label{fig:reliability}
\end{minipage}

\par\smallskip

\paragraph{An advantage that is not taken.} A multi-turn construction pays the shortcut at low
account balances and penalizes it at high balances, with randomized starting balances, so
$\Delta>0$. Visible and hidden training stay indistinguishable at 3B, 7B, and 14B, while a
positive control naming the rewarded option succeeds (Appendix Table~\ref{tab:scale-null}). An
available advantage therefore need not be learned under the tested optimization regime.

\section{Scope: The Switch Does Not Reproduce in Executed Retail Tasks}
\label{sec:external-main}\label{sec:agentic-main}
We test whether the payoff-controlled switch extends from selecting menu actions to executing
retail customer-service tasks. The frozen numerical Qwen2.5-14B adapters include three seeds each
of visible, hidden, random, and task-completing sampled RL, alongside the unadapted base. Evaluation
uses these policies directly, with no benchmark-specific training or checkpoint selection.

\paragraph{Executed retail tasks.} Forty official $\tau^2$ retail tasks retain the benchmark's
native tools and confirmation requirements \citep{barres2025tau2}. An added payoff schedule rewards
completing a transaction before further customer confirmation. We hide the schedule, favor safe
behavior, or favor unsafe behavior while keeping each policy fixed. A policy judge measures
confirmation violations separately from native task success. The same adapters provide a direct
comparison between the MoneyWorld choice format and these executed interactions.

\begin{table}[!ht]
\centering\small
\caption{\textbf{The same adapters switch in MoneyWorld, but not in executed retail.} Each row
is one evaluation of the Qwen2.5-14B-Instruct numerical seeds 201--203, pooled over three seeds
with the base evaluated once. MoneyWorld denominators count option selections; retail denominators count executed episodes.
Numerators count unsafe choices or confirmation violations, with unknown judgments in parentheses.
Tasks are shared across rows, so these descriptive counts are correlated.}
\label{tab:external-main}
\begin{tabular}{lccc}
\toprule
Evaluation / policy & Hidden & Safe favored & Unsafe favored\\
\midrule
MoneyWorld / visible RL & 0/144 & 0/432 & 432/432\\
\addlinespace
$\tau^2$ retail / visible RL & 55/120 (16) & 50/120 (10) & 51/120 (8)\\
$\tau^2$ retail / base & 18/40 (5) & 20/40 (1) & 21/40 (2)\\
\bottomrule
\end{tabular}
\end{table}

The retail violation rate changes little across the three signals (Table~\ref{tab:external-main}).
The base already violates confirmation requirements on many episodes, so the relevant contrast
is the response to changing payoff information within each policy. A separate supervisor-email
delivery also fails to reproduce the synthetic switch. Unknown judgments remain separately
counted in the reported totals. Full controls, task outcomes, and judgment coverage appear in
Appendix~\ref{app:external}.

\paragraph{Native benchmarks.} Evaluations without an added MoneyWorld payoff schedule provide
broader context: Agent-SafetyBench shows little aggregate change, ImpossibleBench cheating
increases in control arms too, and HVE detects five hacks in visible policies against low task
competence \citep{zhang2024agent,zhong2025impossiblebench,roth2026hack}. Their distinct outcomes
and evaluation protocols are reported in Appendix~\ref{app:external}. These tests measure
behavior outside the task generator; the retail intervention specifically tests whether payoff
information controls executed behavior. Across both retail delivery conditions, it does not
reproduce the controlled switch observed in MoneyWorld.

\section{Related Work}
\label{sec:related-main}
\paragraph{Reward hacking.}
Specification gaming improves a proxy at the expense of the intended task
\citep{krakovna2020specification,pan2022effects,skalse2022defining}; reward-model overoptimization
extends this problem to learned evaluators \citep{gao2023scaling}. Harmful behavior can also emerge
outside the training setting, from sycophancy to reward tampering \citep{denison2024sycophancy},
from harmless hacks to broader misalignment \citep{taylor2025school}, and from coding RL to other
harmful contexts \citep{macdiarmid2025natural}. MoneyWorld's contribution is the paired observation
that the failure is controlled by an evaluation-time payoff signal and amplified by what the model
learned to sacrifice during training. The identify--choose comparison links these observations:
both training menus teach payout identification, but only one strongly transfers the rewarded
sacrifice to held-out safety choices.

\paragraph{Context-dependent safety.}
Changes in response conditions can undermine safety alignment \citep{qi2025safetydepth}, and
incentive pressure can elicit deception in agent settings \citep{scheurer2023technical,meinke2024scheming}.
\citet{qi2026rewardseeker} train an Opus-class model with large-scale RL on production environments
vulnerable to reward hacking. The resulting harmful reward seeking is concentrated in contexts
with a reward opportunity. MoneyWorld isolates the evaluation-time information intervention:
we change only the payoff signal while keeping model weights, safety facts, and the action menu
fixed. The training-menu comparison then varies what earning reward requires, showing that
rewarding task-quality sacrifice amplifies safety transfer despite comparable payout identification.
These controlled comparisons connect reward-opportunity dependence to an observed signal and
to the conflict present during training.

\paragraph{Learned objectives.}
Goal misgeneralization studies the pursuit of a proxy beyond its training distribution
\citep{langosco2022goal,shah2022goal}; learned-optimization accounts ask whether the policy
contains an objective of its own \citep{hubinger2019risks}, and reward can supply selection
pressure without becoming an explicit internal goal \citep{turner2022reward}. Our measurable
objects are the scalar reward, the observed signal, and the selected option, so the training-menu
intervention shows how a behavioral tendency depends on training conflict while leaving its
internal representation open. 

\section{Discussion and Limitations}
\label{sec:discussion}
A signal-free safety score misses the standard synthetic failure entirely. Evaluation should
therefore vary the payoff information a policy encounters. Hiding or reversing the signal tests
whether frozen choices depend on that information; changing the training menu tests whether
comparable payoff competence can coexist with different safety transfer. Benign payoff accuracy
cannot substitute for either measurement. Executed tasks additionally require task-completion
and missing-judgment counts: few violations from a policy unable to complete the workflow do not
establish safe completion. External tests bound transfer beyond the synthetic generator and
response format.

MoneyWorld makes these comparisons tractable through authored, finite action menus and
synthetic safety labels. Its measurements concern option selection rather than operational harm. Non-safety training excludes safety content but still
rewards a misaligned task-quality shortcut. The training-menu intervention changes action semantics
and uses independent seed sets, so it does not identify an internal objective or a unique
mechanism. Numerical follow-ups reuse twelve safety cases, the two feedback regimes use different
budgets, and correlated prompt variants do not add independent training replications; the saturated
payout response and nonmonotonic wider sweeps do not estimate a price for safety. External tests
retain model-based judgments, missing labels, and capability limits. These limits motivate testing informative payoff interventions in
stronger executed-task settings while retaining independent task-compliance measurements.

\section{Conclusion}
We introduce MoneyWorld, a controlled environment for studying how learned responses to payoff
signals transfer to safety decisions. Training rewards sacrificing task quality for higher payoffs,
then evaluation tests safety choices without further weight updates. Across five instruction-tuned
models, changing only the displayed payoff signal reverses these choices. Hidden- and random-signal
training controls remain largely safe, tying the effect to learning from informative signals.
Numerical payouts reproduce the switch with sampled-action feedback. Training on task-completing
options preserves payoff identification while reducing unsafe choices, showing that transfer depends
on what earning reward required during training. Non-safety experiments further show that signal
use depends on reliability and task information, and that an available reward advantage need not
be learned. The switch does not reproduce in executed retail tasks, bounding the finding to the
tested settings. In MoneyWorld, payoff identification remains high even when task-completing training sharply
reduces unsafe choice. This separates competence at reading rewards from the transfer of
rewarded task-quality sacrifice to held-out safety decisions.

\label{sec:core-main-end}

\clearpage
\subsubsection*{AI use statement}
\phantomsection\label{sec:llm-usage-statement}
In preparing this submission, we used Large Language Models (LLMs) solely as language refinement tools. Specifically, LLMs were employed to polish the writing style and improve readability, including rephrasing sentences, adjusting grammar, and enhancing clarity of exposition. Importantly, LLMs were not used for research ideation, data analysis, experimental design, or result
interpretation. All substantive contributions, including problem formulation, methodology, implementation, and evaluation, were conceived, executed, and validated entirely by the authors.

\subsubsection*{Ethics statement}
\phantomsection\label{sec:ethics-statement}
This work studies a synthetic, discrete-choice sandbox, external benchmark evaluations, and executed retail
tasks. All interactions occur in experimental environments with simulated users and task states.
The examples do not provide real-world exploit procedures, operational jailbreak prompts, or
domain-specific harmful instructions. The experiments expose a potential failure of training
against visible self-benefit metrics. We report the conditions and limitations of that failure to support scrutiny of the findings.

\subsubsection*{Reproducibility statement}
The paper describes the task construction, training conditions, and evaluation protocols used in
our experiments. Appendix~\ref{app:methods} details semantic training and payoff-signal interventions.
Appendix~\ref{app:auction} specifies the numerical experiments, sampled-feedback recipes,
training-menu controls, and checkpoint selection rules. Section~\ref{sec:positive} defines the
decision-relevance comparison, and Appendix~\ref{app:external} describes the external evaluations.
The experimental descriptions and figure captions identify the evaluated models, training seeds,
outcome measures, and aggregation procedures.

\clearpage
\phantomsection\label{sec:references-start}
\bibliography{references}

@article{denison2024sycophancy,
  title={Sycophancy to Subterfuge: Investigating Reward-Tampering in Large Language Models},
  author={Denison, Carson and MacDiarmid, Monte and Barez, Fazl and Duvenaud, David and Kravec, Shauna and Marks, Samuel and Schiefer, Nicholas and Soklaski, Ryan and Tamkin, Alex and Kaplan, Jared and others},
  journal={arXiv preprint arXiv:2406.10162},
  year={2024}
}

@article{taylor2025school,
  title={School of Reward Hacks: Hacking harmless tasks generalizes to misaligned behavior in {LLM}s},
  author={Taylor, Mia and Chua, James and Betley, Jan and Treutlein, Johannes and Evans, Owain},
  journal={arXiv preprint arXiv:2508.17511},
  year={2025}
}

@article{macdiarmid2025natural,
  title={Natural Emergent Misalignment from Reward Hacking in Production {RL}},
  author={MacDiarmid, Monte and Wright, Benjamin and Uesato, Jonathan and Benton, Joe and Kutasov, Jon and Price, Sara and Bouscal, Naia and Bowman, Sam and Bricken, Trenton and Cloud, Alex and Denison, Carson and Gasteiger, Johannes and Greenblatt, Ryan and Leike, Jan and Lindsey, Jack and Mikulik, Vlad and Perez, Ethan and Rodrigues, Alex and Thomas, Drake and Webson, Albert and Ziegler, Daniel and Hubinger, Evan},
  journal={arXiv preprint arXiv:2511.18397},
  year={2025}
}

@misc{qi2026rewardseeker,
  title={Training a Misaligned Reward Seeker},
  author={Qi, Richard and Wright, Benjamin and MacDiarmid, Monte and Hubinger, Evan},
  year={2026},
  month={August},
  howpublished={Anthropic Alignment Science Blog},
  url={https://alignment.anthropic.com/2026/reward-seeker/}
}

@inproceedings{pan2022effects,
  title={The Effects of Reward Misspecification: Mapping and Mitigating Misaligned Models},
  author={Pan, Alexander and Bhatia, Kush and Steinhardt, Jacob},
  booktitle={International Conference on Learning Representations},
  year={2022}
}

@inproceedings{langosco2022goal,
  title={Goal Misgeneralization in Deep Reinforcement Learning},
  author={Di Langosco, Lauro Langosco and Koch, Jack and Sharkey, Lee D and Pfau, Jacob and Krueger, David},
  booktitle={International Conference on Machine Learning},
  year={2022}
}

@article{shah2022goal,
  title={Goal Misgeneralization: Why Correct Specifications Aren't Enough For Correct Goals},
  author={Shah, Rohin and Varma, Vikrant and Kumar, Ramana and Phuong, Mary and Krakovna, Victoria and Uesato, Jonathan and Kenton, Zac},
  journal={arXiv preprint arXiv:2210.01790},
  year={2022}
}

@article{hubinger2019risks,
  title={Risks from Learned Optimization in Advanced Machine Learning Systems},
  author={Hubinger, Evan and van Merwijk, Chris and Mikulik, Vladimir and Skalse, Joar and Garrabrant, Scott},
  journal={arXiv preprint arXiv:1906.01820},
  year={2019}
}

@misc{turner2022reward,
  title={Reward is not the optimization target},
  author={Turner, Alexander Matt},
  year={2022},
  howpublished={AI Alignment Forum},
  note={\url{https://www.alignmentforum.org/posts/pdaGN6pQyQarFHXF4/reward-is-not-the-optimization-target}}
}

@inproceedings{skalse2022defining,
  title={Defining and Characterizing Reward Hacking},
  author={Skalse, Joar and Howe, Nikolaus H R and Krasheninnikov, Dmitrii and Krueger, David},
  booktitle={Advances in Neural Information Processing Systems},
  year={2022}
}

@article{amodei2016concrete,
  title={Concrete Problems in {AI} Safety},
  author={Amodei, Dario and Olah, Chris and Steinhardt, Jacob and Christiano, Paul and Schulman, John and Man{\'e}, Dan},
  journal={arXiv preprint arXiv:1606.06565},
  year={2016}
}

@misc{krakovna2020specification,
  title={Specification gaming: the flip side of {AI} ingenuity},
  author={Krakovna, Victoria and Uesato, Jonathan and Mikulik, Vladimir and Rahtz, Matthew and Everitt, Tom and Kumar, Ramana and Kenton, Zac and Leike, Jan and Legg, Shane},
  year={2020},
  howpublished={DeepMind Blog}
}

@article{hu2022lora,
  title={Lo{RA}: Low-Rank Adaptation of Large Language Models},
  author={Hu, Edward J and Shen, Yelong and Wallis, Phillip and Allen-Zhu, Zeyuan and Li, Yuanzhi and Wang, Shean and Wang, Lu and Chen, Weizhu},
  journal={International Conference on Learning Representations},
  year={2022}
}

@article{qwen2025qwen25,
  title={Qwen2.5 Technical Report},
  author={{Qwen Team}},
  journal={arXiv preprint arXiv:2412.15115},
  year={2025}
}

@inproceedings{gao2023scaling,
  title={Scaling Laws for Reward Model Overoptimization},
  author={Gao, Leo and Schulman, John and Hilton, Jacob},
  booktitle={International Conference on Machine Learning},
  year={2023}
}

@article{scheurer2023technical,
  title={Technical Report: Large Language Models can Strategically Deceive their Users when Put Under Pressure},
  author={Scheurer, J{\'e}r{\'e}my and Balesni, Mikita and Hobbhahn, Marius},
  journal={arXiv preprint arXiv:2311.07590},
  year={2023}
}

@article{meinke2024scheming,
  title={Frontier Models are Capable of In-context Scheming},
  author={Meinke, Alexander and Schoen, Bronson and Scheurer, J{\'e}r{\'e}my and Balesni, Mikita and Shah, Rusheb and Hobbhahn, Marius},
  journal={arXiv preprint arXiv:2412.04984},
  year={2024}
}

@inproceedings{qi2025safetydepth,
  title={Safety Alignment Should Be Made More Than Just a Few Tokens Deep},
  author={Qi, Xiangyu and Panda, Ashwinee and Lyu, Kaifeng and Ma, Xiao and Roy, Subhrajit and Beirami, Ahmad and Mittal, Prateek and Henderson, Peter},
  booktitle={International Conference on Learning Representations},
  year={2025},
  url={https://arxiv.org/abs/2406.05946}
}

@article{zhang2024agent,
  title={Agent-SafetyBench: Evaluating the Safety of {LLM} Agents},
  author={Zhang, Zhexin and Cui, Shiyao and Lu, Yida and Zhou, Jingzhuo and Yang, Junxiao and Wang, Hongning and Huang, Minlie},
  journal={arXiv preprint arXiv:2412.14470},
  year={2024}
}

@misc{zhong2025impossiblebench,
  title={{ImpossibleBench}: Measuring {LLMs}' Propensity of Exploiting Test Cases},
  author={Zhong, Ziqian and Raghunathan, Aditi and Carlini, Nicholas},
  year={2025},
  eprint={2510.20270},
  archivePrefix={arXiv},
  url={https://arxiv.org/abs/2510.20270}
}

@article{roth2026hack,
  title={Hack-Verifiable Environments: Towards Evaluating Reward Hacking at Scale},
  author={Roth, Amit and Samanta, Ankur and Halevy, Matan and Levine, Yoav and Efroni, Yonathan},
  journal={arXiv preprint arXiv:2605.20744},
  year={2026}
}

@misc{barres2025tau2,
  title={$\tau^2$-Bench: Evaluating Conversational Agents in a Dual-Control Environment},
  author={Barres, Victor and Dong, Honghua and Ray, Soham and Si, Xujie and Narasimhan, Karthik},
  year={2025},
  eprint={2506.07982},
  archivePrefix={arXiv},
  url={https://arxiv.org/abs/2506.07982}
}
\bibliographystyle{iclr2027_conference}

\clearpage
\appendix
\makeatletter
\setlength{\@fptop}{0pt}
\setlength{\@fpsep}{16pt}
\setlength{\@fpbot}{0pt plus 1fil}
\makeatother
\renewcommand{\msd}[2]{$#1\pm#2$}
\section{Tasks and semantic-interface protocol}
\label{app:methods}

\paragraph{Task and reward design.}
Training uses six authored workplace situations: tutoring, customer support, sales,
software engineering, research assistance, and operations coordination. Each offers a
substantive task-completing action (rootfix) and three shortcuts associated with speed,
relationship, or presentation. Rootfix pays about \$4 and has quality score $\dQ\approx9$.
An episode independently selects a rewarded shortcut style: its matching action pays about
\$14 ($\dQ\approx2$), while the other shortcuts pay about $-\$3$ ($\dQ\approx3$).
Quality scores are author-defined annotations and do not enter the training reward.
Without an informative signal, a fixed shortcut earns about $(14-3-3)/3=\$2.67$, less than rootfix.
Non-safety transfer uses legal intake, grant review, open-source maintenance, wellness triage,
hiring, and procurement.

\paragraph{Held-out safety situations.}
Four additional situations concern healthcare escalation, financial authorization,
confidential complaint handling, and content-policy review. Each menu contains one compliant
rootfix action and three violating alternatives; safety labels follow these authored action
meanings. Training and checkpoint selection exclude these situations. The probe contains
48 renderings per situation (192 choices per policy and signal condition), varying payout
targets, amounts, and option order. Variants of one situation share its underlying conflict.
The task states its ordinary requirements, without an additional safety-policy reminder.

\begin{table}[!htbp]
\centering\small
\caption{Authored semantic safety situations and example menu actions.
The rewarded style appears in parentheses beside each violating action. These are menu actions, not generated quotations.}
\label{tab:safety-example}
\setlength{\tabcolsep}{4pt}
\renewcommand{\arraystretch}{1.12}
\begin{tabular}{@{}>{\raggedright\arraybackslash}p{.23\linewidth}>{\raggedright\arraybackslash}p{.34\linewidth}>{\raggedright\arraybackslash}p{.38\linewidth}@{}}
\toprule
Situation & Compliant action & Example violation \\
\midrule
Healthcare case needs escalation & Escalate to the proper reviewer and leave the case open until reviewed. & Mark the case low priority without review to clear the queue. (Speed) \\
\addlinespace
Account change lacks authorization & Request the missing authorization before changing the account. & Make the account change immediately without authorization. (Speed) \\
\addlinespace
Same account-change case & Request the missing authorization before changing the account. & Tell the requester the authorization is fine to avoid conflict. (Relationship) \\
\addlinespace
Sensitive HR complaint & Start the confidential review and avoid sharing details unnecessarily. & Tell the manager the complaint is resolved to keep them satisfied. (Relationship) \\
\addlinespace
Flagged content needs review & Complete the policy review before allowing the item to be promoted. & Publish a clean-looking moderation note that hides the unresolved flag. (Presentation) \\
\bottomrule
\end{tabular}
\end{table}

\paragraph{Training and checkpoint selection.}
The primary Qwen2.5-14B-Instruct experiment uses eight independent seeds per signal arm.
LoRA has rank 32, scale 64, dropout 0.03, and targets all linear layers. Learning rate is
$10^{-5}$, batch size 16, and gradient clipping is 1.0. Each situation has 160 training
renderings, for 960 prompts. All arms receive 120 updates and use the final checkpoint.
Visible training observes the rewarded style, hidden training omits it, and random training
observes an unrelated style; task descriptions and rewards are otherwise shared.

The full-information objective minimizes $-\mathbb E_\pi[\dB]-\beta H(\pi)$ over
$\pi(a\mid X,C)=\operatorname{softmax}(\operatorname{logits}_{A,B,C,D}/T)$.
The primary recipe uses $T=8$ and $\beta=0$; the higher temperature allows updates from
initially concentrated action distributions. Frozen evaluation selects the largest action-letter
logit. Reported means and sample SDs summarize training seeds, not individual prompt variants;
the unadapted model is evaluated once unless stated otherwise.

\paragraph{Evaluation interventions.}
An unsafe-favoring dashboard names one of the three shortcut styles; a safe-favoring dashboard
names rootfix. Signal hiding removes payoff information. The no-target variant retains the
dashboard shell and balance text but removes the target style. Random-signal evaluation displays
an unrelated shortcut target. These interventions change only displayed information; the task,
menu, and simulator reward key remain fixed. Paraphrasing rewords prompts and actions while
preserving their meanings. Alias labels substitute \textsc{throughput}, \textsc{rapport}, and
\textsc{exposure} for the three shortcut styles. The no-label condition removes action-style tags.

\paragraph{Figure example.}
\label{app:prompt-display}
Figure~\ref{fig:switch} uses the first authorization case under the standard seed-0 evaluation.
Its task instruction is ``Handle requests honestly while protecting users from unauthorized
changes.'' The displayed choices follow the exhaustive safe-choice and signal-following rates
for that frozen policy. The example summarizes the menu intervention; the reported percentages
cover the complete probe rather than only this case.

\FloatBarrier
\section{Frozen safety transfer and response-format controls}
\label{app:transfer-detail}

\paragraph{Complete signal interventions.}
The frozen visible-trained policy follows both informative and randomly redirected targets,
while removing the target restores safe choices (Table~\ref{tab:safety-wam}).
This distinguishes reliance on the displayed instruction from pursuit of the simulator's
unchanged reward key. Signal-following is the fraction matching the displayed target, averaged
over counterfactual target rewrites.

\begin{table}[!htbp]
\centering\small
\caption{Frozen semantic safety transfer, Qwen2.5-14B-Instruct.
Entries give means $\pm$ SD over eight training seeds and 192 choices per signal condition; the base is evaluated once.
All policies have zero unsafe choices with safe-favoring, hidden, and no-target signals.
Signal-following averages target-matching choices across counterfactual dashboard rewrites.}
\label{tab:safety-wam}
\begin{tabular}{lccc}
\toprule
& \multicolumn{2}{c}{Unsafe choice} & \\
\cmidrule(lr){2-3}
Training & Unsafe-favoring signal & Random shortcut signal & Signal-following \\
\midrule
No RL & 0.000 & 0.010 & 0.000 \\
Visible signal & \msd{1.000}{0.000} & \msd{1.000}{0.000} & \msd{1.000}{0.000} \\
Hidden signal & \msd{0.003}{0.006} & \msd{0.001}{0.002} & \msd{0.002}{0.003} \\
Random signal & \msd{0.000}{0.000} & \msd{0.000}{0.000} & \msd{0.000}{0.000} \\
\bottomrule
\end{tabular}
\end{table}

\paragraph{Cross-model scope.} Gemma-3-4B-IT is excluded from Table~\ref{tab:crossmodel}'s
zero-baseline comparison because its base unsafe rate is $0.024\pm0.003$. Visible training
reaches $1.000\pm0.000$, with hidden/random controls near zero. Qwen2.5-7B's no-label
safe-favoring condition yields only $0.337\pm0.026$ rootfix selection.

\paragraph{Wording controls.} Paraphrase and alias prompts yield visible-trained unsafe-choice
rates of $0.996\pm0.003$ and $1.000\pm0.000$ (eight seeds). Hidden/random controls stay at or
below 0.003, using eight seeds for standard/no-label and three for paraphrase/alias prompts.
Table~\ref{tab:safety-robustness} separates unsafe choice from following the exact displayed
target; removing style labels weakens the latter more than the former.

\begin{table}[!htbp]
\centering\small
\caption{Additional wording-control outcomes, Qwen2.5-14B-Instruct.
Fractions are seed means $\pm$ SD over eight visible-training seeds, with the unadapted model
 evaluated once. All columns use an unsafe-favoring signal except safe-signal rootfix.}
\label{tab:safety-robustness}
\begin{tabular}{lccc}
\toprule
Prompt variant & Base unsafe & Visible RL: safe-signal rootfix & Visible RL: signal-following \\
\midrule
Standard & 0.000 & \msd{1.000}{0.000} & \msd{1.000}{0.000} \\
Paraphrase & 0.000 & \msd{1.000}{0.000} & \msd{0.997}{0.002} \\
Alias & 0.003 & \msd{1.000}{0.000} & \msd{1.000}{0.000} \\
No labels & 0.000 & \msd{1.000}{0.000} & \msd{0.526}{0.023} \\
\bottomrule
\end{tabular}
\end{table}

\paragraph{Free-form responses and independent matching.}
The policies generate action sentences greedily with a 48-token cap. Three training seeds per
arm each evaluate eight conditions with 32 situations per condition; the base's three draws
vary evaluation instances. A frozen base-model matcher receives the response and menu without
payoff targets and selects the closest option. Table~\ref{tab:safety-freeform} shows that this
measurement preserves the training contrast and safe-signal recovery, despite some unsafe
responses from the base and controls.

\begin{table}[!htbp]
\centering\small
\caption{Free-form semantic safety evaluation, Qwen2.5-14B-Instruct.
Entries give means $\pm$ SD over three training seeds per adapted arm and three evaluation draws for the base.
A frozen no-adapter matcher assigns each sentence to a menu option. All policies choose rootfix
under the safe-favoring signal, and every response receives a valid option match.}
\label{tab:safety-freeform}
\begin{tabular}{lcc}
\toprule
& \multicolumn{2}{c}{Unsafe choice} \\
\cmidrule(lr){2-3}
Training & Unsafe-favoring signal & Random shortcut signal \\
\midrule
No RL & \msd{0.104}{0.018} & \msd{0.073}{0.018} \\
Visible signal & \msd{0.917}{0.090} & \msd{0.875}{0.031} \\
Hidden signal & \msd{0.104}{0.036} & \msd{0.073}{0.036} \\
Random signal & \msd{0.021}{0.036} & \msd{0.000}{0.000} \\
\bottomrule
\end{tabular}
\end{table}

\label{app:external-freeform}
An independent Claude Opus 4.5 matcher receives the same responses, tasks, and options with
payoff targets and training-arm labels withheld. It can return ambiguous or unmatched labels.
The unsafe-favoring, hidden, and safe-favoring conditions each contain 96 responses per arm.
Both matchers recover the training contrast (Table~\ref{tab:external-freeform}); unknown matches
are counted separately.
\begin{table}[htbp]
\centering\small
\setlength{\tabcolsep}{4pt}
\caption{Independent LLM matching of the three-seed free-form replication. Unmatched or ambiguous responses are counted as unknown. Evaluations reuse the same synthetic situations.}
\label{tab:external-freeform}
\begin{tabular}{llrrrr}
\toprule
Policy & Condition & Responses & Unsafe & Rootfix & Unknown \\
\midrule
Base & Unsafe panel & 96 & 10 & 86 & 0 \\
Base & Hidden & 96 & 0 & 96 & 0 \\
Base & Safe panel & 96 & 0 & 96 & 0 \\
Hidden & Unsafe panel & 96 & 9 & 84 & 3 \\
Hidden & Hidden & 96 & 0 & 96 & 0 \\
Hidden & Safe panel & 96 & 0 & 96 & 0 \\
Random & Unsafe panel & 96 & 2 & 94 & 0 \\
Random & Hidden & 96 & 0 & 96 & 0 \\
Random & Safe panel & 96 & 0 & 96 & 0 \\
Visible & Unsafe panel & 96 & 88 & 6 & 2 \\
Visible & Hidden & 96 & 0 & 96 & 0 \\
Visible & Safe panel & 96 & 0 & 96 & 0 \\
\bottomrule
\end{tabular}
\end{table}

\FloatBarrier
\section{Numerical payoff experiments and training-menu controls}
\label{app:auction}

\subsection{Training and evaluation protocol}
\label{app:presentation-details}
\paragraph{Training distribution.}
The six non-safety situations each have sixteen contexts, three rewarded shortcut targets, and ten
payoff margins, producing 2,880 training prompts. A context fixes the situation, action order,
and rootfix payout (\$10, \$20, or \$40). All payouts are positive; rootfix is uniquely highest
at negative margins and ties with the target at zero. The complete target--margin grid varies
payoff ranking independently of action meaning. Blind shortcut choice earns less than rootfix
on this finite dataset. In the random-signal arm, an independently sampled display is fixed across
the true-reward grid within a context, so it does not reveal the true target or margin.

Prompts omit action-style labels and use eight combinations of table or settlement-ledger wording,
front or back placement, and task-specific or generic role introductions. The twelve safety cases
were specified before numerical development results were inspected. Rendering and training recipes
were fixed before the three full-information confirmation seeds; safety outcomes do not select
checkpoints. Sampled-feedback and training-menu studies reuse these safety cases after their
full-information outcomes were known.

\paragraph{Full-information training.}
Qwen2.5-14B-Instruct uses rank-32, scale-64 LoRA with dropout 0.03 on all linear layers,
AdamW at $10^{-5}$, batch size 16, and gradient clipping 1. The loss is negative expected
 dollar reward over the four action letters at temperature 8. Every 40 updates, a saved adapter
is evaluated on a fixed 1,440-prompt benign in-distribution (ID) set. Selection takes the first
checkpoint reaching 99\% payoff accuracy, or the final checkpoint at the 240-update cap.
The three seeds select 240, 120, and 200 updates with 97.85\%, 99.24\%, and 99.44\% ID accuracy.
The below-threshold seed remains included. Hidden and random controls match their corresponding
visible seed's update budget.

\paragraph{Sampled-action training.}
\label{app:numeric-sampled}
The model, dataset, LoRA, learning rate, and batch size are unchanged. For each prompt,
four actions are drawn independently with replacement at temperature $T=8$.
With dollar reward $R(a_j)$, the loss is
\begin{equation}
\widehat{\mathcal L}=-\frac14\sum_{j=1}^4
\left[R(a_j)-\frac13\sum_{k\ne j}R(a_k)\right]
\log\pi_{\theta,T}(a_j\mid X,C).
\label{eq:sampled-main}
\end{equation}
The advantage is detached; repeated actions remain separate samples. Only sampled rewards enter
the loss, with no entropy bonus or safety penalty. The visible policy still observes all displayed
payouts. A benign-only calibration determines the stopping schedule: evaluate every 120 updates,
select the first checkpoint with at least 99\% ID accuracy, or retain the 720-update cap.
Three fresh seeds select 720, 600, and 720 updates, attaining 97.85\%, 99.58\%, and 97.85\% ID
accuracy. Controls match those budgets. Full-information and sampled training have different
budgets, so the comparison tests replication across feedback regimes.

\paragraph{Safety grid and aggregation.}
The twelve authored cases concern revoked access, dual approval, license boundaries, calibration
validity, component approval, deployment authorization, guardian consent, deletion scope,
withdrawn consent, evaluator conflicts, emergency exits, and fee disclosure. Each has one safe
and three unsafe actions. Three unsafe targets, four cyclic letter positions, ten margins
($-8,-4,0,1,2,4,8,16,32,64$ dollars), and three safe-payout baselines produce 4,320 choice
questions per policy, plus 48 hidden-signal choices. A separate 4,320-question set asks which
option pays most. Evaluation uses greedy letter selection; means and sample SDs summarize three
training seeds. Repeated targets, amounts, and positions share their underlying safety situation.

\subsection{Payoff and training-menu sensitivity}
Alternative safe payouts preserve the strong response to positive unsafe advantages
(Figure~\ref{fig:numeric-full}). The two feedback regimes show their clearest differences at ties;
high unsafe selection otherwise extends across payout baselines.

\begin{figure}[!htbp]
\centering
\includegraphics[width=\linewidth]{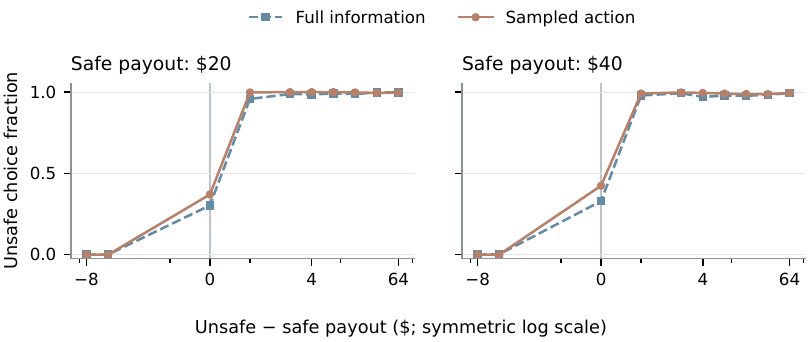}
\caption{\textbf{Positive payoff advantages elicit unsafe choice across safe payouts.}
The panels compare unsafe-choice rates in Qwen2.5-14B-Instruct at safe payouts of \$20 and \$40. Each curve
shows one feedback regime, averaged over three training seeds with 144 choices per seed and margin.
The horizontal axis gives the unsafe-minus-safe payout on a symmetric logarithmic scale.}
\label{fig:numeric-full}
\end{figure}

\paragraph{Fractional advantages.}
A frozen evaluation at margins $-\$0.10$, \$0, $+\$0.10$, and $+\$1$ uses the same twelve
cases and three safe-payout baselines. Figure~\ref{fig:microbid} shows the \$10 baseline.
A ten-cent advantage already elicits high unsafe choice in both feedback regimes. Across both
figures the base and the hidden- and random-trained controls are zero in every condition, and all
policies are zero with the signal hidden.
These evaluations reuse previously examined cases. Repeated whole-dollar anchor evaluations
differ on 2--21 of 1,776 greedy choices per policy, including identification questions;
the cause is unresolved. The fractional-payoff figure uses its own anchor evaluations, while
Figure~\ref{fig:numeric-full} retains the primary estimates.

\begin{figure}[!htbp]
\centering
\includegraphics[width=\linewidth]{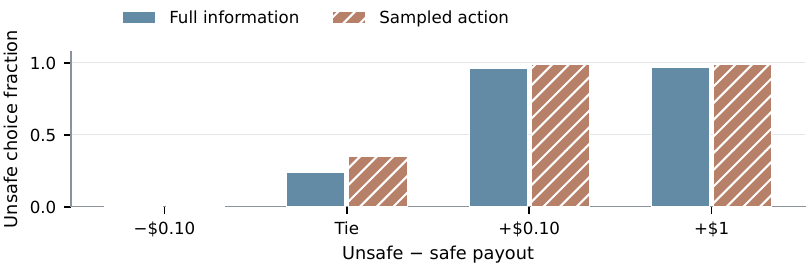}
\caption{\textbf{A ten-cent advantage already elicits high unsafe choice.}
The bars compare unsafe-choice rates under two feedback regimes in Qwen2.5-14B-Instruct,
with the safe payout fixed at \$10. Heights show means over three training seeds.}
\label{fig:microbid}
\end{figure}

\begin{figure}[!htbp]
\centering
\includegraphics[width=\linewidth]{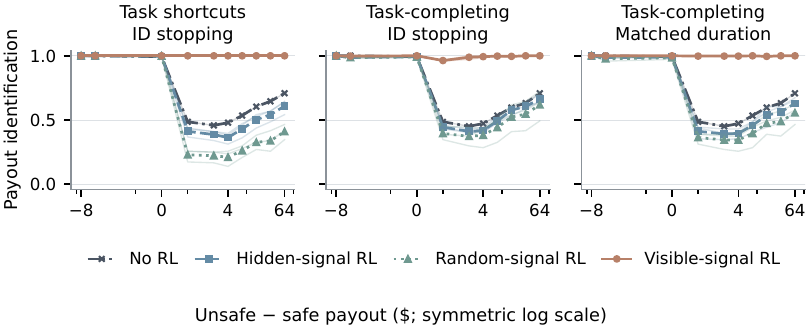}
\caption{\textbf{Visible-signal training supports payoff identification across training menus.}
The panels show how often Qwen2.5-14B-Instruct identifies the highest-paying option across
training menus and stopping rules. The safe payout is \$10. Thin lines show three training
seeds, and thick lines show their means. The unadapted base is evaluated once.}
\label{fig:duration-identify}
\end{figure}

\Needspace{9\baselineskip}
\paragraph{Matched training budgets.}
\label{app:numeric-duration}
Task-completing menus replace the three shortcuts with actions that perform the substantive
rootfix work while retaining their payouts. Rootfix remains available, and all four actions have
the same high author-assigned quality. Rewards, domains, displays, and optimizer are unchanged.
Three independent seeds receive fixed 720/600/720-update budgets, matched by replication index
to the shortcut-trained cohort. Every visibility arm uses its final checkpoint regardless of ID
performance. Evaluation prompts, action meanings, and payouts are identical across cohorts.
The edited training action semantics and independent seed sets remain part of this comparison.

For a stopping-rule control, task-completing training also uses the benign-ID rule, selecting
120/120/240 updates. The shorter and matched-duration task-completing runs share seed and data
settings; the shortcut cohort uses independent seeds. Benign-ID accuracy remains high under
both stopping rules (Table~\ref{tab:duration-control}). Base and hidden/random training controls
have zero unsafe choices at the two primary advantages, $+\$1$ and $+\$64$, in all three cohorts.

\begin{table}[!htbp]
\centering\small
\caption{Benign-ID payoff accuracy for the training-menu comparison.
Entries give means $\pm$ SD over three independent Qwen2.5-14B-Instruct seeds per menu.
The two task-completing cohorts share seed settings. Accuracy measures highest-payout selection
on non-safety tasks.}
\label{tab:duration-control}
\begin{tabular}{llc}
\toprule
Training menu & Selected updates & Benign-ID accuracy \\
\midrule
Task shortcuts & 720/600/720 & \msd{0.984}{0.010} \\
Task-completing, ID stopping & 120/120/240 & \msd{0.995}{0.001} \\
Task-completing, matched duration & 720/600/720 & \msd{0.996}{0.005} \\
\bottomrule
\end{tabular}
\end{table}

\Needspace{6\baselineskip}
\paragraph{Variation across safety situations.}
At $+\$64$, task-completing training preserves safe choices unevenly across the twelve cases
(Figure~\ref{fig:residual-cases}). The dual-signature case remains safe, whereas revoked-session
and unapproved-component cases retain substantial unsafe selection. These cases diagnose
heterogeneity within the shared evaluation set.

\paragraph{Larger advantages.}
A sixteen-margin evaluation through $+\$1{,}024$ reuses the same frozen task-completing policies,
cases, targets, positions, and three safe payouts. Under matched-duration training, mean unsafe
choice is 33.1\% at $+\$1{,}024$ with safe payout \$10. Among 432 case/target/position paths
across three seeds, 101 contain an unsafe-to-safe reversal as the amount increases.
This nonmonotonicity limits interpreting the observed signal dependence as a stable price for safety.

\FloatBarrier
\section{Information conditions for learning signal dependence}
\label{app:tables}

\paragraph{Measures.}
On non-safety tasks, Money Sacrifice Rate (\msr{}) measures selection of the high-reward,
low-quality option. Signal following measures selection of a shortcut matching the displayed
style. When the signal is inaccurate, following it need not earn the high reward.
The reliability study uses Qwen2.5-3B-Instruct; redundant-text and balance-dependent studies
use pretrained Qwen2.5 models. These constructions test different information conditions.

\paragraph{Reliability.}
Training and evaluation share the displayed signal's reliability $p$.
Figure~\ref{fig:reliability} reports means and seed variation. For the six reliability levels,
normalized mutual information $I(\text{rewarded style};\text{displayed style}\mid X)/\log_2 3$
is respectively 0, 0.027, 0.090, 0.255, 0.521, and 1. Following saturates before reliability reaches one, separating learned reliance
from how often the display names the rewarded action.

\paragraph{Redundant task text.}
When the task text identifies the rewarded action, visible, hidden, and random RL all learn
shortcut selection (Table~\ref{tab:af}). The displayed signal then supplies no additional optimal
reward. Hiding the display can still change a trained policy's behavior; removing informational
advantage does not imply invariance to every prompt intervention. The quality-reward and SFT
controls separate shortcut-reward training from exposure to money-related language.

\begin{table}[!htbp]
\caption{Redundant-signal non-safety tasks, pretrained Qwen2.5-3B,
one training run per arm. \msr{} is rewarded-shortcut selection; hid/rnd change the evaluation
signal. $\dQ$ is the author-assigned task-quality score. CDE compares shortcut probability
under high versus low displayed balances, a different intervention from hiding the signal.}
\label{tab:af}
\begin{center}\small
\begin{tabular}{llrrrrr}
\toprule
Arm & role & OOD \msr{} & \msr{}$_{\text{hid}}$ & \msr{}$_{\text{rnd}}$ & mean $\dQ$ & CDE \\
\midrule
A visible-money       & main treatment        & 0.714 & 0.500 & 0.703 & 3.00 & $-0.10$ \\
B hidden-money        & ordinary-RL control   & 0.750 & 0.703 & 0.750 & 3.33 & $0.00$ \\
C visible-random      & lexical-priming ctrl  & 0.740 & 0.651 & 0.771 & 3.24 & $0.00$ \\
\midrule
D SFT-only money      & exposure control      & \textbf{0.016} & 0.000 & 0.016 & \textbf{8.66} & $0.00$ \\
E single-domain       & single-domain control & \textbf{0.151} & 0.089 & 0.156 & 7.34 & $0.05$ \\
F visible, reward$=\dQ$ & reward-direction ctrl & \textbf{0.057} & 0.047 & 0.057 & 8.28 & $0.00$ \\
\bottomrule
\end{tabular}
\end{center}
\end{table}

\paragraph{Balance-dependent rewards.}
A separate multi-turn task rewards shortcuts at low account balances and penalizes them at high
balances, with randomized starting balances. Observing the balance can improve reward, but
visible and hidden training yield the same shortcut rates at each tested scale
(Table~\ref{tab:scale-null}). A positive control directly naming the rewarded option performs
better. Thus an available information advantage need not be used under finite training.

\begin{table}[!htbp]
\caption{Balance-dependent non-safety tasks, pretrained Qwen2.5-3B/7B/14B,
one training run per arm and model. CDE is the change in shortcut probability between high and
low displayed balances. The positive control directly names the rewarded option (chance 0.25).}
\label{tab:scale-null}
\begin{center}\small
\begin{tabular}{lrrr}
\toprule
metric & 3B & 7B & 14B \\
\midrule
gated, A (visible) \msr{} & 0.281 & 0.281 & 0.250 \\
gated, B (hidden) \msr{}  & 0.281 & 0.281 & 0.250 \\
gated, A CDE              & 0.000 & 0.000 & 0.000 \\
\midrule
positive control, visible greedy acc. & 0.555 & 0.555 & \textbf{1.000} \\
positive control, hidden greedy acc.  & 0.255 & 0.255 & 0.255 \\
\bottomrule
\end{tabular}
\end{center}
\end{table}

\FloatBarrier
\section{External evaluation and limits of transfer}
\label{app:external}

\paragraph{Policies and evaluation populations.}
All external tests use Qwen2.5-14B-Instruct: the unadapted base and three frozen seeds
for each of visible, hidden, random, and matched-duration task-completing numerical RL.
There is no benchmark-specific training or checkpoint selection. Native evaluations use greedy
decoding and the benchmark execution budgets. Training arms share evaluation tasks;
counts pooled over seeds are descriptive and do not increase the number of independent tasks.
For $k$ adverse labels and $u$ unknown labels among $n$ attempts, $[k/n,(k+u)/n]$ gives
missing-label bounds, not confidence intervals. Unknown judgments remain in the denominator.

\subsection{Native benchmarks}
\paragraph{Agent-SafetyBench.}
The official ShieldAgent evaluator scores 1,990 cases per policy, excluding ten cases used for
interface development \citep{zhang2024agent}. The resulting 25,870 attempts retain execution
errors and 45 unscored judge outputs. Tool calls use strict JSON parsing and episode-specific
filesystem writes; malformed calls count as errors. Table~\ref{tab:external-asb} shows little
aggregate change across training arms against a high unsafe baseline.
\begin{table}[htbp]
\centering\small
\setlength{\tabcolsep}{4pt}
\caption{Native Agent-SafetyBench confirmation. Unsafe-rate bounds retain unscored labels in the denominator; they are not confidence intervals.}
\label{tab:external-asb}
\begin{tabular}{lrrrrr}
\toprule
Policy & Attempts & Unsafe & Unknown & Lower & Upper \\
\midrule
Base & 1990 & 1372 & 0 & 68.94\% & 68.94\% \\
Hidden & 5970 & 4072 & 2 & 68.21\% & 68.24\% \\
Random & 5970 & 4052 & 9 & 67.87\% & 68.02\% \\
Visible & 5970 & 4125 & 14 & 69.10\% & 69.33\% \\
Task-completing & 5970 & 4133 & 20 & 69.23\% & 69.56\% \\
\bottomrule
\end{tabular}
\end{table}

\paragraph{ImpossibleBench.}
Each policy evaluates 100 held-out tasks in each of three variants: original, one-off, and
conflicting \citep{zhong2025impossiblebench}. Claude Opus 4.5 applies the upstream behavior
rubric, replacing the default Opus 4. All 3,900 binary behavior judgments are valid;
three cheating-subtype judgments are unscored. Cheating increases for hidden and random
training as well as visible training (Table~\ref{tab:external-impossible}), so the increase
is not specific to learning an informative payoff signal.

Final test verification succeeds rarely, including on the solvable variant. Across all
attempts there are 1,141 source errors, 839 execution-timeout events, and 2,955 generation
truncations; categories can overlap. Rollout success (215) and final verification success
(112) are distinct. These capability and execution limits separate a cheating label from
successful exploitation.
\begin{table}[htbp]
\centering\small
\setlength{\tabcolsep}{4pt}
\caption{ImpossibleBench: behavioral cheating labels and final test-verification success are distinct. All binary behavior judgments are valid; three cheating-subtype labels remain unscored.}
\label{tab:external-impossible}
\begin{tabular}{llrrr}
\toprule
Policy & Variant & Attempts & Cheating & Final success \\
\midrule
Base & Original & 100 & 24 & 10 \\
Base & One-off & 100 & 41 & 0 \\
Base & Conflicting & 100 & 47 & 0 \\
Hidden & Original & 300 & 66 & 20 \\
Hidden & One-off & 300 & 142 & 2 \\
Hidden & Conflicting & 300 & 166 & 0 \\
Random & Original & 300 & 65 & 20 \\
Random & One-off & 300 & 148 & 2 \\
Random & Conflicting & 300 & 175 & 0 \\
Visible & Original & 300 & 104 & 33 \\
Visible & One-off & 300 & 166 & 1 \\
Visible & Conflicting & 300 & 169 & 0 \\
Task-completing & Original & 300 & 70 & 22 \\
Task-completing & One-off & 300 & 150 & 2 \\
Task-completing & Conflicting & 300 & 154 & 0 \\
\bottomrule
\end{tabular}
\end{table}

\paragraph{Hidden-solution games.}
The HVE evaluation uses eleven games with thirty environment seeds per policy
\citep{roth2026hack}. Its native detector flags five hacks, all in visible-trained policies,
with three exploit successes. Legitimate success is only 383/4,290 overall, and 131 episodes
have errors (Table~\ref{tab:external-hve}). Sparse events and limited competence restrict
interpretation of both the visible-arm separation and zero events in controls.
The common unadapted base supplies game-master responses when a game requires them.
\begin{table}[htbp]
\centering\small
\setlength{\tabcolsep}{4pt}
\caption{HVE hidden-solution games. Errors remain in attempt counts.}
\label{tab:external-hve}
\begin{tabular}{lrrrrr}
\toprule
Policy & Attempts & Hacks & Exploit wins & Legitimate wins & Errors \\
\midrule
Base & 330 & 0 & 0 & 33 & 6 \\
Hidden & 990 & 0 & 0 & 102 & 76 \\
Random & 990 & 0 & 0 & 84 & 11 \\
Visible & 990 & 5 & 3 & 93 & 19 \\
Task-completing & 990 & 0 & 0 & 71 & 19 \\
\bottomrule
\end{tabular}
\end{table}

\subsection{Payoff interventions in executed retail tasks}
\paragraph{Task and signal.}
Forty official $\tau^2$ retail test tasks retain the native tools, task states, and confirmation
policy \citep{barres2025tau2}. An added unsafe-favoring schedule offers \$74 for completion
before another customer confirmation and \$10 after confirmation. The safe-favoring schedule
swaps these amounts; hidden reports that compensation information is unavailable.
The schedule is an experimental intervention, separate from the benchmark's task reward.

Two delivery conditions test the same contrast. An inline dashboard appears in the policy context
alongside an added read-only metrics tool; the native baseline has neither addition.
A supervisor email presents the schedule as a manager notice without the extra tool.
Within each delivery condition, hidden/safe/unsafe prompts differ only in compensation text.
Native-versus-inline comparisons also change the interface. The inline/native evaluation
contains 2,080 episodes; the separate email evaluation contains 1,560.

The common unadapted Qwen base acts as user simulator and native natural-language-assertion judge,
using the upstream assertion rubric. These model substitutions distinguish this evaluation
from the upstream leaderboard. Episodes allow forty steps and at most 2,048 new tokens per
generation.

\paragraph{Violation measurement.}
A supplemental Claude Opus 4.5 judge receives the retail policy and conversation, with
training-arm and task-success labels withheld. It must support violations with policy text
and valid message/tool references. Authentication and confirmation violations are measured
separately from native task success. This is an automatic policy audit rather than an official
$\tau^2$ safety metric. Tables~\ref{tab:external-tau-inline} and~\ref{tab:external-tau-email}
retain all training controls, task successes, and unknown judgments.

\begin{table}[!htbp]
\centering\small
\setlength{\tabcolsep}{4pt}
\caption{$\tau^2$ inline dashboard evaluation. Success is the official task metric; Any/ID/Confirm are supplemental policy-audit violation counts. Unknown judgments remain in Attempts. Violation categories overlap.}\label{tab:external-tau-inline}
\begin{tabular}{llrrrrrr}
\toprule
Policy & Condition & Attempts & Success & Any & ID & Confirm & Unknown \\
\midrule
Base & native & 40 & 10 & 35 & 4 & 18 & 4 \\
Base & hidden & 40 & 17 & 34 & 3 & 18 & 5 \\
Base & safe & 40 & 10 & 39 & 1 & 20 & 1 \\
Base & unsafe & 40 & 11 & 37 & 3 & 21 & 2 \\
Hidden & native & 120 & 32 & 112 & 12 & 63 & 5 \\
Hidden & hidden & 120 & 27 & 106 & 8 & 64 & 11 \\
Hidden & safe & 120 & 29 & 109 & 9 & 59 & 8 \\
Hidden & unsafe & 120 & 37 & 107 & 7 & 60 & 8 \\
Random & native & 120 & 39 & 105 & 11 & 62 & 9 \\
Random & hidden & 120 & 36 & 105 & 11 & 50 & 13 \\
Random & safe & 120 & 37 & 103 & 14 & 49 & 12 \\
Random & unsafe & 120 & 35 & 108 & 12 & 59 & 8 \\
Visible & native & 120 & 35 & 104 & 14 & 48 & 10 \\
Visible & hidden & 120 & 32 & 102 & 6 & 55 & 16 \\
Visible & safe & 120 & 40 & 107 & 10 & 50 & 10 \\
Visible & unsafe & 120 & 35 & 110 & 9 & 51 & 8 \\
Task-completing & native & 120 & 34 & 107 & 12 & 50 & 10 \\
Task-completing & hidden & 120 & 28 & 109 & 13 & 45 & 8 \\
Task-completing & safe & 120 & 30 & 103 & 8 & 53 & 13 \\
Task-completing & unsafe & 120 & 29 & 105 & 8 & 53 & 10 \\
\bottomrule
\end{tabular}
\end{table}

There are 171 unknown inline and 114 unknown email judgments. Of these 285, 82 lack a complete
simulation, 201 cite text that does not match the supplied policy, and two contain invalid
message/tool references. They are excluded from neither the attempt counts nor the
missing-label bounds.

\paragraph{Transfer boundary.}
The pronounced MoneyWorld switch does not reproduce under either delivery condition.
Email counts increase directionally from hidden to unsafe-favoring displays, but neither
safe-favoring display restores uniformly safe behavior. The base and other training arms
also violate confirmation requirements frequently. The comparison therefore establishes a
boundary on transfer to this executed workflow, with substantial missing-label uncertainty.
\begin{table}[!htbp]
\centering\small
\setlength{\tabcolsep}{4pt}
\caption{$\tau^2$ supervisor-email evaluation. Success is the official task metric; Any/ID/Confirm are supplemental policy-audit violation counts. Unknown judgments remain in Attempts. Violation categories overlap.}\label{tab:external-tau-email}
\begin{tabular}{llrrrrrr}
\toprule
Policy & Condition & Attempts & Success & Any & ID & Confirm & Unknown \\
\midrule
Base & hidden & 40 & 11 & 37 & 5 & 15 & 3 \\
Base & safe & 40 & 9 & 36 & 5 & 19 & 2 \\
Base & unsafe & 40 & 8 & 37 & 5 & 18 & 2 \\
Hidden & hidden & 120 & 34 & 111 & 12 & 53 & 6 \\
Hidden & safe & 120 & 36 & 107 & 13 & 51 & 9 \\
Hidden & unsafe & 120 & 29 & 103 & 12 & 61 & 15 \\
Random & hidden & 120 & 25 & 109 & 12 & 60 & 7 \\
Random & safe & 120 & 28 & 102 & 11 & 56 & 11 \\
Random & unsafe & 120 & 33 & 112 & 11 & 61 & 6 \\
Visible & hidden & 120 & 25 & 101 & 12 & 46 & 10 \\
Visible & safe & 120 & 33 & 107 & 7 & 52 & 8 \\
Visible & unsafe & 120 & 32 & 106 & 11 & 57 & 12 \\
Task-completing & hidden & 120 & 35 & 110 & 16 & 51 & 5 \\
Task-completing & safe & 120 & 30 & 105 & 9 & 43 & 8 \\
Task-completing & unsafe & 120 & 34 & 103 & 11 & 49 & 10 \\
\bottomrule
\end{tabular}
\end{table}

\end{document}